\pdfoutput=1

\documentclass[a4paper,fleqn]{cas-sc}

\usepackage[numbers]{natbib}
\usepackage{forest}

\def\tsc#1{\csdef{#1}{\textsc{\lowercase{#1}}\xspace}}
\tsc{WGM}
\tsc{QE}
\tsc{EP}
\tsc{PMS}
\tsc{BEC}
\tsc{DE}
\usepackage{float}
\usepackage{caption}

\usepackage{amsthm}

\usepackage{tikz}
\usepackage{pgfplots}
\pgfplotsset{compat=1.18}

\usepackage{placeins}

\usepackage[noend]{algpseudocode}
\usepackage{algorithm}

\newtheorem{definition}{Definition}
\newtheorem{example}{Example}

\newtheorem{theorem}{Theorem}
\newtheorem{lemma}[theorem]{Lemma}

\begin{document}
\let\WriteBookmarks\relax
\def\floatpagepagefraction{1}
\def\textpagefraction{.001}

\shorttitle{Beyond Explaining Predictions: Logic-Based Explanations for Confidence in Machine Learning Models}


\title[mode = title]{Beyond Explaining Predictions: Logic-Based Explanations for Confidence in Machine Learning Models}                      



%
\author{Vinícius Peixoto Chagas}[
                        ]



\ead{vinicius.peixoto.chagas61@aluno.ifce.edu.br}




\author{Carlos Henrique Leitão Cavalcante}
\ead{henriqueleitao@ifce.edu.br}



\author{Thiago Alves Rocha}
    [%
   ]
\ead{thiago.alves@ifce.edu.br}

\begin{abstract}
Machine learning is increasingly used in critical domains, where both predictions and their associated confidence levels influence important decisions. To enhance transparency in such scenarios, it is important to understand why a model is confident or uncertain about its predictions. Recent logic-based approaches provide abductive explanations, minimal subsets of features sufficient to preserve the predicted class, with correctness guarantees. However, these methods focus solely on classification behavior and may produce explanations that cover instances with low predictive confidence. In this work, we introduce the concept of Minimum Confidence Threshold (MCT), which quantifies the weakest confidence guarantee provided by an abductive explanation. Building upon this concept, we propose confidence-aware abductive explanations, which preserve not only the predicted class but also a user-specified confidence guarantee. We formulate MCT computation as an optimization problem and introduce an algorithm for generating minimal explanations that satisfy a desired confidence threshold. We evaluate the proposed framework on boosted trees for binary classification, although the approach is applicable to other machine learning models that provide confidence scores. Experimental results show that traditional abductive explanations often provide substantially weaker confidence guarantees than the confidence associated with the explained instance itself. In contrast, confidence-aware explanations consistently improve the minimum confidence guaranteed by an explanation while requiring only a modest increase in explanation length. These properties make the proposed approach particularly suitable for applications where both predictive correctness and confidence are essential for trustworthy decision making.
\end{abstract}




\maketitle

\section{Introduction}

Machine learning (ML) models have seen a significant rise in popularity and application across various fields, including healthcare, finance, and commerce \cite{rudin2019stopexplainingblackbox}. These applications often require not only accurate predictions but also interpretability \cite{ribeiro2016whyitrustyou, ribeiro2018anchors, ignatiev2019onrelating} and reliability \cite{guo2017calibrationmodernneuralnetworks}. In such contexts, it is not enough for models to be accurate — they are also expected to indicate when their predictions may be unreliable. Reliability, therefore, goes beyond accuracy and includes the ability to recognize and express uncertainty.

Logic-based \emph{abductive explanations} stand out as an approach to interpretability, as they provide guarantees of correctness while using a minimal set of features~\cite{ignatiev2019abduction,xreason22,audemard2023computing,bottieau2024logic,gomes2025bound,doncenco2025dive}. For a given instance, an abductive explanation specifies a set of conditions satisfied by that instance such that any other instance satisfying the same conditions is guaranteed to receive the same classification from the model. Furthermore, minimality means that the explanation contains no redundant features, and removing any one of them would invalidate the guarantee of correctness.

Although abductive explanations ensure that the predicted class remains unchanged, they do not take into account the confidence associated with the predictions. As a result, explanations may cover instances for which the model assigns low confidence values, indicating uncertainty. In high-stakes applications, this limitation can reduce the usefulness of the explanation, as it may lead users to believe that all justified predictions are equally reliable, or cause overconfidence in explanations that may not consider the model's limitations.

To overcome this issue, we propose an extension of abductive explanations that incorporates a confidence threshold. While traditional abductive explanations guarantee only the preservation of the predicted class, our extension additionally enforces that all instances satisfying the explanation achieve a predictive confidence no lower than a specified threshold. Therefore, the goal is to compute a minimal subset of features that guarantees both correctness and a desired level of predictive certainty.

We introduce the notion of a minimum confidence threshold for abductive explanations, defined as the lowest predictive confidence among all instances satisfying a given explanation. Based on this notion, we propose a confidence-aware extension of abductive explanations that enforces user-defined confidence thresholds during the explanation generation process. Then, our goal is is to identify minimal subsets of features that ensure both correctness and a desired level of predictive confidence. Our results demonstrate a significant improvement in the minimum confidence guaranteed by the explanations when compared to traditional abductive explanations, while requiring only a small increase in explanation length.

Recent work has increasingly explored the integration of uncertainty into explainable artificial intelligence (XAI) \cite{saini2022select, lofstrom2024calibrated, zhang2025your, lofstrom2026concerning}. In this context, uncertainty can play different roles in the explanation process, including communicating the reliability of explanations or guiding the generation of explanations. For instance, \cite{lofstrom2024calibrated} combines local feature-importance explanations with uncertainty intervals for both predictions and feature contributions, allowing users to assess the reliability of the explanation itself. Other approaches incorporate uncertainty into the explanation generation process~\cite{saini2022select}, using uncertainty estimates to improve the fidelity of local surrogate explanations.

However, most existing uncertainty-aware explanation methods focus on attribution-based, perturbation-based, or probabilistic explanations, without providing formal guarantees regarding the set of instances covered by an explanation. In particular, current approaches typically estimate uncertainty associated with explanations, rather than constraining explanations to guarantee a minimum confidence level across all instances satisfying them. In contrast, our work introduces confidence-awareness directly into abductive explanations while preserving their guarantees of correctness and minimality.





In this work, we focus on gradient-boosted trees (GBTs)~\cite{xgboost}, although our method could be replicable for other machine learning models that provide confidence estimates for their predictions. These models achieve state-of-the-art performance across a wide range of supervised learning tasks~\cite{borisov2022deep}. However, despite their predictive power, boosted trees lack inherent interpretability~\cite{audemard2023computing}, which motivates the development of explanation methods for these models.

This work is organized into seven sections. Section~\ref{sec:background} presents background on GBTs, their interpretability, classification probability, and logic-based explanation methods, noting that such explanations may overlook model confidence. Section~\ref{sec:conf_margin} introduces the Minimum Confidence Threshold (MCT), and Section~\ref{sec:how_increase_confidence} describes how to compute explanations under MCT constraints. Section~\ref{sec:methodology} details the implementation, Section~\ref{sec:results} presents the experimental results, and the final section summarizes the contributions, highlights the benefits of enforcing the MCT for more reliable explanations, and outlines future directions for related studies.





\section{Background} 
\label{sec:background}

This section presents the foundations required for the proposed confidence-aware abductive explanations. First, we introduce gradient-boosted trees (GBTs) and discuss how their outputs can be interpreted as predictive confidence scores. Next, we present first-order logic over linear real arithmetic and focus on the concepts used in the computation of abductive explanations with guarantees of correctness and minimality. Finally, we discuss a key limitation of traditional abductive explanations: although they preserve the predicted class, they do not account for the confidence associated with the model's predictions.



\subsection{Machine Learning and Gradient-Boosted Trees}\label{sec:gbts}

\paragraph{Machine Learning and Binary Classification Problems.} In machine learning, binary classification problems are defined over a set of features $\mathcal{F} = \{x_1, ..., x_n\}$ and a set of two classes $\mathcal{K} = \{-1, +1\}$. An instance or specific data point is a feature assignment $\mathbf{x} = \{x_1 = v_1, \ldots, x_n = v_n\}$, such that each $v_i \in \mathbb{R}$. A classifier $C$ is a function that maps elements from the feature space into the set of classes $\mathcal{K}$. For example, $C$ may map the instance $\{x_1 = v_1, ..., x_n = v_n\}$ to class $+1$. Usually, the classifier is obtained through a training process over a dataset.




\paragraph{Gradient-Boosted Trees.} In this work, we focus on gradient-boosted trees (GBTs) as a classifier $C$. We use the XGBoost classifier~\cite{xgboost}, a widely adopted implementation of GBTs known for its efficiency and predictive performance. A GBT is an ensemble of regression trees $R_1,\ldots,R_m$. Given an instance $\mathbf{x}$, each tree produces a real-valued output $R_h(\mathbf{x})$, and the model aggregates these values into a single prediction score
\[
S(\mathbf{x}) = \sum_{h=1}^{m} R_h(\mathbf{x}) + \textit{init},
\]

where \textit{init} is the initial bias term learned during training. The predicted class is then defined by he sign of the prediction score
\[
C(\mathbf{x}) =
\begin{cases}
+1 & \text{if } S(\mathbf{x}) \geq 0,\\
-1 & \text{otherwise}.
\end{cases}
\]

Besides determining the predicted class, the prediction score also provides information about the model's predictive confidence. Intuitively, scores farther from the decision boundary at $0$ correspond to more confident predictions. A confidence value in the interval $[0,1]$ can be obtained by applying the sigmoid function to the prediction score:
\[
p(\mathbf{x}) = \frac{1}{1 + e^{-S(\mathbf{x})}}.
\]

\noindent The resulting value can be interpreted as the confidence assigned to class $+1$. Therefore, for instances classified as $+1$, confidence increases as $p(\mathbf{x})$ approaches $1$, whereas for instances classified as $-1$, confidence increases as $p(\mathbf{x})$ approaches $0$.


\begin{example}\label{ex:gbt}
To illustrate how prediction scores are computed, consider the example GBT shown in Figure~\ref{fig:iris_trees}. The model consists of two regression trees over features from the Iris dataset\footnote{https://archive.ics.uci.edu/dataset/53/iris}: \textit{Sepal Width} (SW), \textit{Sepal Length} (SL), \textit{Petal Length} (PL), and \textit{Petal Width} (PW). We consider a binary classification task in which the class \textit{Iris Setosa} is represented by $-1$ and all remaining Iris species are grouped into the class \textit{Others}, represented by $+1$. Since feature SL does not appear in any decision path, its value has no effect on the model's prediction and can vary freely without changing the computed score.

Consider the instance $\mathbf{x} = \{ SL = 2, SW = 2, PL = 2.6, and PW = 1.5 \}$. The instance follows the left branch in both trees, producing outputs of $-3$ and $-2$, respectively. For simplicity, we assume $\textit{init}=0$ in this illustrative example. Consequently, $S(\mathbf{x}) = -3 + (-2) = -5.$ Applying the sigmoid function yields a probability of approximately $p(\mathbf{x}) = 0.0067$ for the positive class (\textit{Others}). Therefore, the instance is classified as \textit{Iris Setosa} (class $-1$) with high confidence.    
\end{example}

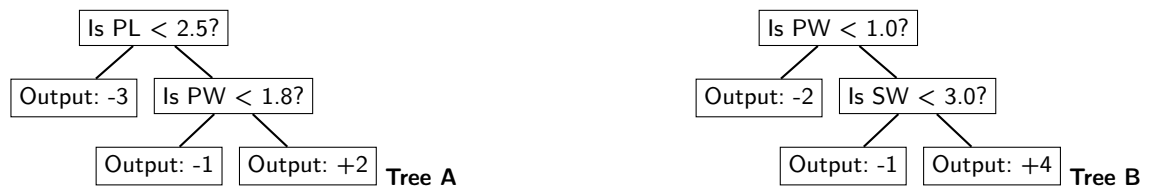
\begin{figure}[h]
  \centering
  \begin{minipage}[b]{0.45\textwidth}
    \centering
    \begin{forest}
      for tree={fill=white, draw, text centered, font=\sffamily, edge={thick, -{[]}, draw}}
        [Is PL < 2.5?
          [Output: -3]
          [Is PW < 1.8?
              [Output: -1]
            [Output: +2]
          ]
        ]
    \end{forest}
    \textbf{Tree A}
  \end{minipage}
  \hfill
  \begin{minipage}[b]{0.45\textwidth}
    \centering
    \begin{forest}
      for tree={fill=white, draw, text centered, font=\sffamily, edge={thick, -{[]}, draw}}
        [Is PW < 1.0?
          [Output: -2]
          [Is SW < 3.0?
            [Output: -1]
            [Output: +4]
          ]
        ]
    \end{forest}
    \textbf{Tree B}
  \end{minipage}
  \caption{Example of a GBT composed of two trees using features from the Iris dataset. Left branches correspond to \textit{False} outcomes and right branches correspond to \textit{True} outcomes.}
  \label{fig:iris_trees}
\end{figure}

\FloatBarrier



Although tree ensembles are often considered more interpretable than models such as deep neural networks, understanding the reasoning behind individual predictions remains challenging. Several explanation methods have been proposed to explain the predictions of machine learning models, including LIME~\cite{ribeiro2016whyitrustyou}, SHAP~\cite{shap}, and ANCHORS~\cite{ribeiro2018anchors}. These methods can also be applied to GBTs to provide insights into the features influencing a prediction.

These methods may provide useful insights into model behavior, but they do not guarantee that the generated explanations faithfully characterize the model's decision process. In particular, an explanation may identify features as relevant even when other instances satisfying the same explanation receive different predictions~\cite{gosiewska2020trustadditiveexplanations}.

To address this limitation, logic-based explanation methods~\cite{xreason22,bjorner2023formal,bottieau2024logic,jemaa2024extendingxreasonformalexplanations,izza2025most,rocha2025generalizing,boumazouza2026fame,wu2026efficiently} exploit formal reasoning techniques to generate explanations with correctness guarantees. Such approaches rely on logical representations of the model and its predictions, allowing explanation generation to be formulated as a reasoning task. The next section introduces the first-order logic concepts used throughout this work, after which we present logic-based abductive explanations for GBTs.

\subsection{First-order Logic over LRA} \label{subsec:logic}

In this work, we use first-order logic (FOL) to give explanations with guarantees of correctness. We use quantifier-free first-order formulas over the theory of linear real arithmetic (LRA). Then, first-order variables are allowed to take values from the real numbers $\mathbb{R}$. For details, see \cite{kroening2016decision}. Therefore, we consider formulas as defined below:
\begin{equation}
        \begin{aligned}
             F, G &:= a \mid (F \wedge G) \mid (F \vee G) \mid (\neg F) \mid (F \to G),\\
             a &:= \sum^n_{i=1} w_i x_i \leq b, 
        \end{aligned}    
\end{equation}
such that $F$ and $G$ are quantifier-free first-order formulas over the theory of linear real arithmetic. Moreover, $a$ represents the atomic formulas such that $n \geq 1$, each $w_i$ and $b$ are fixed real numbers, and each $x_i$ is a first-order variable. For example, $(2.5x_1 + 3.1x_2 \geq 6) \wedge (x_1=1 \vee x_1=2) \wedge (x_1=2 \to x_2 \leq 1.1)$ is a formula by this definition. Observe that we allow standard abbreviations as $\neg (2.5x_1 + 3.1x_2 < 6)$ for $2.5x_1 + 3.1x_2 \geq 6$. 

Since we are assuming the semantics of formulas over the domain of real numbers, an \textit{assignment} $\mathcal{A}$ for a formula $F$ is a mapping from the first-order variables of $F$ to elements in the domain of real numbers. For instance, $\{x_1 \mapsto 2.3, x_2 \mapsto 1\}$ is an assignment for $(2.5x_1 + 3.1x_2 \geq 6) \wedge (x_1=1 \vee x_1=2) \wedge (x_1=2 \to x_2 \leq 1.1)$. An assignment $\mathcal{A}$ \textit{satisfies} a formula $F$ if $F$ is true under this assignment. For example, $\{x_1 \mapsto 2, x_2 \mapsto 1.05\}$ satisfies the formula in the above example, whereas $\{x_1 \mapsto 2.3, x_2 \mapsto 1\}$ does not satisfy it. Moreover, an assignment $\mathcal{A}$ \textit{satisfies} a set $\Gamma$ of formulas if all formulas in $\Gamma$ are true under $\mathcal{A}$.

A set of formulas $\Gamma$ is \textit{satisfiable} if there exists a satisfying assignment for $\Gamma$. To give an example, the set $\{(2.5x_1 + 3.1x_2 \geq 6), (x_1=1 \vee x_1=2), (x_1=2 \to x_2 \leq 1.1)\}$ is satisfiable since $\{x_1 \mapsto 2, x_2 \mapsto 1.05\}$ satisfies it. As another example, the set $\{(x_1 \geq 2), (x_1 < 1)\}$ is unsatisfiable since no assignment satisfies it. Given a set of formulas $\Gamma$ and a formula $G$, the notation $\Gamma \models G$ is used to denote \textit{logical consequence} or \textit{entailment}, i.e., each assignment that satisfies $\Gamma$ also satisfies $G$. As an illustrative example, let $\Gamma = \{x_1 = 2 , x_2 \geq 1\}$ and $G = (2.5x_1 + x_2 \geq 5) \wedge (x_1=1 \vee x_1=2)$. Then, $\Gamma \models G$. The essence of entailment lies in ensuring the correctness of the conclusion $G$ based on the set of premises $\Gamma$. In the context of computing explanations, as presented in \cite{ignatiev2019abduction}, logical consequence serves as a fundamental tool for guaranteeing correctness.
 
The relationship between satisfiability and entailment is a fundamental aspect of logic. It is widely known that, for all sets of formulas $\Gamma$ and all formulas $G$, it holds that $\Gamma \models G$ iff $\Gamma \cup \{\neg G\}$ is unsatisfiable. For instance, $\{x_1 = 2, x_2 \geq 1),  \neg((2.5x_1 + x_2 \geq 5) \wedge (x_1=1 \vee x_1=2))\}$ has no satisfying assignment since an assignment that satisfies $\{x_1 = 2 , x_2 \geq 1\}$ also satisfies $(2.5x_1 + x_2 \geq 5) \wedge (x_1=1 \vee x_1=2)$ and, therefore, does not satisfy $\neg((2.5x_1 + x_2 \geq 5) \wedge (x_1=1 \vee x_1=2))$. Since our approach builds upon the concept of logical consequence, we can leverage this connection in the context of computing explanations.

\subsection{Abductive Explanations for GBTs} \label{sec:logic_based_xai}
 
Logic-based explanation methods provide a rigorous alternative to heuristic explanation techniques by offering formal guarantees of correctness and minimality~\cite{xreason22, ignatiev2019abduction}. \emph{Abductive Explanations} (AXps) identify minimal subsets of feature assignments that are sufficient to guarantee a model's prediction. Intuitively, once the values of the features appearing in an abductive explanation are fixed, the predicted class remains unchanged regardless of the values assigned to the remaining features.

The notion of correctness in abductive explanations is therefore defined with respect to the classifier's output class. In particular, an explanation guarantees that every instance satisfying its feature assignments receives the same classification as the original instance. Moreover, the requirement of minimality ensures that the explanation contains no redundant information: removing any feature from the explanation would invalidate this guarantee. Consequently, AXps provide concise explanations that are both correct and irredundant~\cite{ignatiev2019abduction}. Formally, we define an AXp as follows:

\begin{definition}[Abductive Explanation \cite{ignatiev2019abduction}]\label{defi-axp}
Let $\mathbf{x} = \{x_1 = v_1, ..., x_n = v_n\}$ be an instance and $C$ be a classifier such that $C(\mathbf{x}) \in \mathcal{K}$. An \emph{abductive explanation} $\mathcal{E}$ is a minimal subset of $\mathbf{x}$ such that for all $v_1', ..., v_n'$, if $v'_j = v_j$ for each $x_j = v_j \in E$, then $C(\{x_1 = v'_1, ..., x_n = v'_n\}) = C(\mathbf{x})$.
\end{definition}

In other words, an AXp identifies key features that are sufficient for the output. Minimality ensures that the explanation $\mathcal{E}$ does not include any redundant features. In other words, removing any feature $x_j = v_j$ from $\mathcal{E}$ would result in a subset that no longer guarantees the same prediction $C(\mathbf{x})$. 

Observe that Definition~\ref{defi-axp} guarantees only the preservation of the predicted class through the condition $C(\{x_1 = v'_1, ..., x_n = v'_n\}) = C(\mathbf{x})$. Therefore, any instance satisfying the explanation is required to receive the same classification as the original instance. However, the definition does not impose any constraint on the confidence associated with those predictions. As a result, an AXp may simultaneously cover highly confident and highly uncertain instances, provided that they are assigned the same class. This limitation motivates the confidence-aware extensions proposed later in this work.

To compute abductive explanations for GBTs, logic-based approaches represent the instance, the model, and the prediction as first-order formulas over LRA~\cite{ignatiev2019validating}. This logical representation enables explanation generation to be formulated as a reasoning task based on the notions of satisfiability and entailment introduced in Section~\ref{subsec:logic}.

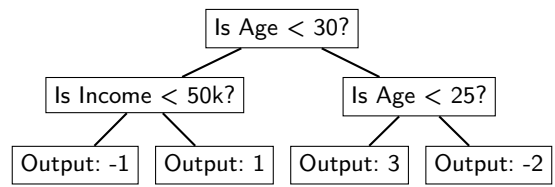
\begin{figure}[ht]
  \centering
  \begin{forest}
    for tree={fill=white, draw, text centered, font=\sffamily, edge={thick, -{[]}, draw}}
      [Is Age < 30?
        [Is Income < 50k?
          [Output: -1]
          [Output: 1]
        ]
        [Is Age < 25?
          [Output: 3]
          [Output: -2]
        ]
      ]
  \end{forest}
  \caption{Example of a tree in a GBT model (left branches correspond to \textit{False} outcomes, and right branches correspond to \textit{True} outcomes).}
  \label{fig:decision_tree}
\end{figure}
\FloatBarrier

The first step consists of encoding the GBT model as a formula $G$. Since each tree $R$ is composed of decision paths defined by threshold conditions, every path can be naturally represented as an implication whose antecedent describes the path conditions and whose consequent specifies the corresponding leaf value. For example, consider the tree shown in Figure~\ref{fig:decision_tree}. Its behavior can be represented by the following set of formulas:
\begin{equation}
\label{eq:example_texp}
\left\{
\begin{array}{ll}
\text{Age} \geq 30 \land \text{Income} \geq 50k \rightarrow o_1 = -1, \\
\text{Age} \geq 30 \land \text{Income} < 50k \rightarrow o_1 = 1, \\
\text{Age} < 30 \land \text{Age} \geq 25 \rightarrow o_1 = 3, \\
\text{Age} < 25 \rightarrow o_1 = -2
\end{array}
\right\}
\end{equation}

The variables in the formulas are Age, Income, and $o_1$, where $o_1$ represents the output of the first tree. More generally, the complete ensemble can be encoded as the following formula:

\begin{equation}
    G :=  \bigwedge_{h = 1}^m \left[ \bigwedge_{P \in R_h}\left(\bigwedge_{f_j \odot l_{P, j} \in P} (f_j \odot l_{P, j}) \rightarrow o_h = l_{P}\right) \right]
    \label{eq:texp}
\end{equation}

where $m$ is the number of trees, $P$ denotes a path in tree $R_h$, $x_j$ is a feature occurring in path $P$, $\odot \in \{<,\geq\}$ is the comparison operator associated with the branch followed in the path, $l_{P,j}$ is the threshold associated with feature $x_j$, $o_h$ is the variable representing the output of tree $R_h$, and $l_P$ is the leaf value reached by path $P$.

Once the model is encoded, the prediction itself is represented as a logical formula. Recall from Section~\ref{sec:gbts} that the prediction score is obtained by summing the outputs of all trees and the bias term. Moreover, a positive sum indicates class $1$ and a negative sum indicates class $-1$. Therefore, the decision of the model can be encoded as the formula:

\begin{equation}
    D = 
    \begin{cases} 
        \sum_{h=1}^m o_h + \text{init} \geq 0, & \text{if } C(\mathbf{x}) = +1 \\
        \sum_{h=1}^m o_h + \text{init} < 0, & \text{if } C(\mathbf{x}) = -1.
    \end{cases}
    \label{eq:dexp}
\end{equation}

The formulas $G$ and $D$ allow explanation generation to be formulated as an entailment problem. Given an instance $\mathcal{x}$ such that $C(\mathbf{x}) = +1$, a subset $\mathcal{X} \subseteq \mathbf{x}$ is sufficient to entail the prediction if $\mathcal{X} \cup \{G\} \models \sum_{h=1}^m o_h + \text{init} \geq 0$.

As discussed in Section~\ref{subsec:logic}, entailment can be checked through satisfiability. Therefore, instead of directly proving the entailment above, we may test whether the set of formulas
\begin{equation}
\label{eq-axp-ent}
\mathcal{X} \cup \{G\} \cup \{\neg \mathcal{D}\}
\end{equation}

\noindent is unsatisfiable. This formula has a natural interpretation. It asks whether there exists an assignment of values to the features that is consistent with the model encoding $G$, satisfies the currently selected feature assignments in $\mathcal{X}$, and simultaneously violates the original prediction through $\neg\mathcal{D}$. Such an assignment would constitute a counterexample to the explanation.



Abductive explanations are computed through an iterative feature-removal process. Starting from the complete instance $\mathbf{x}$, features are removed one at a time and the satisfiability of \eqref{eq-axp-ent} is checked. If the formula remains unsatisfiable after removing a feature, then no counterexample exists and the feature is considered unnecessary. Otherwise, a satisfying assignment exists, meaning that there is a way to change the values of the features not fixed by the explanation so that the predicted class changes. In this case, the removed feature remains in the explanation. The final explanation consists of all features whose removal could change the classification outcome. By construction, the explanation is correct because it guarantees the original classification for every instance satisfying its conditions, and it is minimal because removing any remaining feature would invalidate this guarantee. Algorithm~\ref{alg:axp} summarizes the feature-removal procedure used to compute an abductive explanation.

\begin{algorithm}
\caption{Computing an Abductive Explanation}
\label{alg:axp}
\begin{algorithmic}[1]
\Require{$\mathbf{x}$, $G$, $\mathcal{D}$}
\Ensure{Abductive explanation $\mathcal{E}$}
\State $\mathcal{E} \gets \mathbf{x}$
\For{$(f_i=v_i) \in \mathbf{x}$}
    \State $\mathcal{X} \gets \mathcal{E} \setminus \{f_i=v_i\}$
    \If{$\mathcal{X} \cup \{G\} \cup \{\neg\mathcal{D}\}$ is unsatisfiable}
        \State $\mathcal{E} \gets \mathcal{X}$
    \EndIf
\EndFor
\State \Return $\mathcal{E}$
\end{algorithmic}
\end{algorithm}

Although abductive explanations provide strong guarantees of correctness and minimality, these guarantees are restricted to the predicted class. They do not characterize how confident the model is across the set of instances covered by the explanation. Therefore, two instances satisfying the same explanation may receive the same classification while exhibiting substantially different confidence levels. The next section explores this limitation and lays the foundations for confidence-aware abductive explanations, in which guarantees are extended beyond the predicted class to also consider predictive confidence.

\section{Limitation of AXps and Minimum Confidence Threshold of Explanations}
\label{sec:conf_margin}

Abductive explanations guarantee that all instances satisfying the explanation receive the same class as the original instance~\cite{ignatiev2019validating}. However, this guarantee is defined solely in terms of class preservation. No restriction is imposed on the prediction score associated with those instances. Therefore, AXps treat all decisions as equally reliable.

As a result, a single abductive explanation may simultaneously cover highly confident and highly uncertain predictions. Although all covered instances receive the same class, some of them may lie much closer to the decision boundary than others. Consequently, an explanation can provide a correct justification for a prediction while offering no guarantees regarding the confidence associated with the set of instances satisfying it.

To illustrate this limitation, we return to the binary classification task based on the Iris dataset introduced in Example~\ref{ex:gbt}, where flowers are classified as either \textit{Iris Setosa} (class $-1$) or \textit{Others} (class $+1$). 

\begin{example}\label{ex:twoconfi}
Consider the following two instances:
\begin{align}
\mathbf{x_1} &= \{\,\text{SL} = 5.1,\ \text{SW} = 3.5,\ 
                 \text{PL} = 1.4,\ \text{PW} = 0.2\,\} \notag \\
\mathbf{x_2} &= \{\,\text{SL} = 5.0,\ \text{SW} = 3.6,\ 
                 \text{PL} = 1.4,\ \text{PW} = 0.2\,\} \notag
\end{align}

Assume a trained machine learning model that classifies both $\mathbf{x_1}$ and $\mathbf{x_2}$ as \textit{Iris Setosa}, but $p(\mathbf{x_1}) = 0.02$, while $p(\mathbf{x_2}) = 0.49$. This indicates uncertainty from the model in the second case. An abductive explanation that considers only the predicted class might produce the following explanation for $\mathbf{x_1}$:
\begin{align}
\{\text{PL} = 1.4, \text{PW} = 0.2 \} \notag
\end{align}

Since the explanation is abductive, it guarantees that any instance with those values is classified as \textit{Iris Setosa}. However, this explanation is based only on the model's prediction class and does not consider the associated confidence of the model. The problem is that it also applies to $\mathbf{x_2}$, even though the model is uncertain about that classification.     
\end{example}

The fact that this explanation covers both a high-confidence and a low-confidence instance highlights a critical flaw: it conveys a level of certainty that the model does not possess. Such explanations can cause overconfidence in automated decisions and may obscure the model’s actual limitations and risks.


Example~\ref{ex:twoconfi} illustrates that two instances satisfying the same explanation may be associated with very different confidence levels. We now revisit the GBT model from Figure~\ref{fig:iris_trees} to examine this phenomenon more concretely. In particular, we show that an abductive explanation may either be associated with a single prediction score for all instances satisfying it or admit multiple possible prediction scores while still guaranteeing the same class.

\begin{example}\label{ex:axp_scores}

Consider the instances $\mathbf{x_3}$ and $\mathbf{x_4}$ in Table~\ref{tab:abductive_examples} and their respective abductive explanations. The abductive explanation for $\mathbf{x_3}$ guarantees a score of $-5$, as \textit{Tree A} will output $-3$ and \textit{Tree B} will output $-2$, resulting in the classification \textit{Setosa}. In this case, features SW and SL may assume any values without affecting either the predicted class or the resulting score. Consequently, every instance satisfying this explanation is associated with the same prediction score of $-5$.
\vspace{\intextsep} %
\begin{center} 
\captionof{table}{Examples of abductive explanations for two instances.} \label{tab:abductive_examples}
\begin{tabular}{|l|l|l|}
\hline
\textbf{Instance} & \textbf{Score / Class} & \textbf{Abductive Explanation} \\
\hline
$\mathbf{x_3} = \{$SL = 2, SW = 2, PL = 2.6, PW = 1.1$\}$ & -5 / \textit{Setosa} & $\mathcal{E}_{\mathbf{x_3}} = \{$PL = 2.6, PW = 1.1$\}$ \\
$\mathbf{x_4} = \{$SL = 2, SW = 2, PL = 3.0, PW = 1.9$\}$ & -3 / \textit{Setosa}           & $\mathcal{E}_{\mathbf{x_4}} = \{$PW = 1.9$\}$  \\
\hline
\end{tabular}
\end{center}
\vspace{\intextsep} 


Now consider the abductive explanation for $\mathbf{x_4}$. It guarantees the predicted class as the remaining features do not affect the classification of \textit{Setosa}. However, the prediction score can vary due to the influence of these remaining features. In \textit{Tree A}, PW = 1.9 leads to two possible outputs: $-3$ and $-1$, depending on the value of PL. While in \textit{Tree B}, the output is fixed at $-2$. Combining the trees, there are two possible prediction scores: $-5$ and $-3$. Although both scores correspond to the class \textit{Setosa}, they are associated with different confidence levels. Applying the sigmoid function yields probabilities of approximately $0.0067$ and $0.0474$ for the positive class (\textit{Others}), respectively. Consequently, the score $-3$ is considerably closer to the decision boundary than $-5$ and therefore represents the least confident prediction among the instances satisfying this explanation.

\end{example}

Although this example produces only two possible prediction scores, real GBT models may contain hundreds of trees and much deeper decision paths. In such models, the features left unspecified by an abductive explanation can influence many tree traversals, resulting in a large number of possible prediction scores for the instances satisfying the explanation. Consequently, the confidence associated with an explanation cannot, in general, be characterized by a single prediction score of the original instance. Instead, it is necessary to consider the range of scores that may arise among all instances satisfying the explanation.

This observation motivates the need for a confidence measure associated with the explanation itself rather than with a particular instance. Intuitively, the most relevant quantity is the least confident prediction that can still satisfy the explanation, since it represents the weakest guarantee provided by the explanation.

Given this, we define the \textbf{Minimum Confidence Threshold} (MCT) of an explanation as the lowest model confidence (i.e., the prediction score closest to the decision boundary) among all instances satisfying an abductive explanation. Intuitively, the MCT represents the worst predictive confidence that the model can assign to any instance satisfying the explanation. The formal definition is presented below.

\begin{definition}[Minimum Confidence Threshold]
Let $\mathbf{x}$ be an instance, let $\mathcal{E}$ be an abductive explanation for $\mathbf{x}$, and let $S(\mathbf{x})$ denote the prediction score produced by the model. The \textbf{Minimum Confidence Threshold} (MCT) of $\mathcal{E}$ is the prediction score closest to the decision boundary among all instances satisfying $\mathcal{E}$:

\[
\mathrm{MCT}(\mathcal{E})=
\begin{cases}
\displaystyle \min_{\mathcal{E}\subseteq \mathbf{x}'} S(\mathbf{x}')
& \text{if } C(\mathbf{x})=+1,\\[2ex]
\displaystyle \max_{\mathcal{E}\subseteq \mathbf{x}'} S(\mathbf{x}')
& \text{if } C(\mathbf{x})=-1.
\end{cases}
\]

where the subset condition $\mathcal{E}\subseteq \mathbf{x}'$ indicates that all feature assignments appearing in the explanation $\mathcal{E}$ are also present in the instance $\mathbf{x}'$, i.e., $\mathbf{x}'$ satisfies the conditions specified by the explanation.
\end{definition}

Observe that the MCT corresponds to the worst-case prediction score among all instances satisfying the explanation. Therefore, higher MCT values indicate stronger confidence guarantees for explanations of class $+1$, whereas lower MCT values indicate stronger confidence guarantees for explanations of class $-1$.
\begin{example}\label{ex:mct}
Consider again the abductive explanations for $\mathbf{x_3}$ and $\mathbf{x_4}$ presented in Example~\ref{ex:axp_scores}. For $\mathbf{x_3}$, every instance satisfying the explanation produces the same prediction score of $-5$. Therefore,
\[
\mathrm{MCT}(\mathcal{E}_{\mathbf{x_3}}) = -5.
\]

For $\mathbf{x_4}$, the explanation admits two possible prediction scores, namely $-5$ and $-3$, depending on the values assigned to the remaining features. Since $\mathbf{x_4}$ belongs to class $-1$, the MCT is the score closest to the decision boundary, that is,
\[
\mathrm{MCT}(\mathcal{E}_{\mathbf{x_4}})
=
\max\{-5,-3\}
=
-3.
\]

This example illustrates that an abductive explanation can be associated with multiple prediction scores, while the MCT provides a single value that summarizes the weakest confidence guarantee offered by that explanation.
\end{example}

In this work, we define the MCT in terms of prediction scores rather than probabilities, since prediction scores are directly computed by the GBT model and can be naturally incorporated into the formal reasoning procedures used to generate explanations. Nevertheless, the MCT can be directly converted into a confidence value in the interval $[0,1]$. In this interpretation, applying the sigmoid function to the MCT yields the confidence associated with the least confident prediction among all instances satisfying the explanation. For example, the explanation of $\mathbf{x_4}$ in Example~\ref{ex:mct} has $\mathrm{MCT}(\mathcal{E}_{x_4})=-3$. Applying the sigmoid function yields $\sigma(-3)\approx 0.0474$, which corresponds to a probability of approximately $4.74\%$ for class $+1$ (\textit{Others}) and, equivalently, $95.26\%$ for class $-1$ (\textit{Setosa}). Therefore, although some instances satisfying the explanation may achieve higher confidence, the explanation guarantees at least $95.26\%$ confidence for the class \textit{Setosa}.

The notion of MCT naturally leads to a confidence-aware extension of abductive explanations. While a traditional abductive explanation guarantees only the preservation of the predicted class, a confidence-aware abductive explanation additionally guarantees that every instance satisfying the explanation achieves a minimum confidence level.

\begin{definition}[Confidence-Aware Abductive Explanation]\label{def:caAXp}
Let $\mathbf{x}$ be an instance, let $C(\mathbf{x})$ be its predicted class, and let $\tau$ be a prediction score threshold. An abductive explanation $\mathcal{E}$ for $\mathbf{x}$ is \emph{confidence-aware} with respect to $\tau$ if it is an abductive explanation and:
\begin{equation}\label{eq:caAXP}
\begin{cases}
\mathrm{MCT}(\mathcal{E}) \leq \tau, & \text{if } C(\mathbf{x})=-1,\\[1ex]
\mathrm{MCT}(\mathcal{E}) \geq \tau, & \text{if } C(\mathbf{x})=+1.
\end{cases}
\end{equation}

Furthermore, $\mathcal{E}$ is subset-minimal among the explanations satisfying the corresponding constraint.
\end{definition}

The MCT provides a quantitative characterization of the confidence guarantees associated with an abductive explanation, while confidence-aware abductive explanations allow these guarantees to be explicitly controlled through user-defined thresholds $\tau$.

Traditional abductive explanations can be interpreted as a special case of this formulation. Recall that an abductive explanation guarantees that all instances satisfying the explanation receive the same class as the original instance. In terms of MCT, this condition is equivalent to requiring that the worst-case prediction score remains on the same side of the decision boundary as the original prediction. Therefore,
\[
\mathrm{MCT}(\mathcal{E}) \geq 0
\]
must hold for explanations of class $+1$, whereas
\[
\mathrm{MCT}(\mathcal{E}) < 0
\]
must hold for explanations of class $-1$. In other words, traditional abductive explanations correspond to using the decision boundary itself as the confidence threshold. This observation naturally leads to confidence-aware abductive explanations, where the boundary value $0$ is replaced by a user-defined threshold $\tau$.

The choice of $\tau$ is naturally constrained by the prediction score of the instance being explained. Recall that the original instance itself satisfies its abductive explanation. Therefore, the MCT of an explanation cannot be more confident than the prediction score of the explained instance.

For example, suppose an instance $\mathbf{x}$ is classified as class $+1$ with prediction score $S(\mathbf{x})=3$. Since $\mathbf{x}$ satisfies any abductive explanation generated for it, no explanation can achieve $\mathrm{MCT}(\mathcal{E})>3$. Consequently, it would be impossible to require a confidence-aware abductive explanation with threshold $\tau=4$. Similarly, if an instance is classified as class $-1$ with prediction score $S(\mathbf{x})=-3$, no explanation can satisfy $\mathrm{MCT}(\mathcal{E})<-3$. As a result, valid threshold values must satisfy
\[
\begin{cases}
\tau \leq S(\mathbf{x}), & \text{if } C(\mathbf{x})=+1,\\[1ex]
\tau \geq S(\mathbf{x}), & \text{if } C(\mathbf{x})=-1.
\end{cases}
\]

The prediction score of the explained instance therefore defines the strongest confidence guarantee that any confidence-aware abductive explanation can provide.

Although we define $\tau$ in terms of prediction scores, it can also be interpreted in terms of predictive confidence. Since the sigmoid function is monotonic, every score threshold corresponds to a unique probability threshold. For example, requiring $\tau=-3$ for an explanation of a class $-1$ instance is equivalent to requiring that every instance satisfying the explanation has a probability of at most $\sigma(-3)\approx0.0474$ for class $+1$, or equivalently a probability of at least $95.26\%$ for class $-1$. Similarly, requiring $\tau=3$ for an explanation of a class $+1$ instance guarantees a probability of at least $\sigma(3)\approx0.9526$ for class $+1$.

Therefore, confidence-aware abductive explanations may be specified either through score thresholds or through confidence thresholds in the interval $[0,1]$. In this work, we adopt prediction scores because they are directly produced by the model and naturally align with the formal reasoning procedures used to generate explanations.

In the next section, we present the adaptations required to compute confidence-aware abductive explanations. In particular, we show how the satisfiability-based feature-removal procedure can be extended to compute the MCT of an explanation and enforce minimum confidence thresholds during explanation generation.





\section{Computing Abductive Explanations with Confidence Guarantees}
\label{sec:how_increase_confidence}

The previous section introduced the notion of Minimum Confidence Threshold (MCT) and showed that traditional abductive explanations guarantee only class preservation, without constraining the confidence associated with the instances satisfying the explanation. In this section, we extend the computation of abductive explanations to incorporate confidence guarantees.

More specifically, we first show how the MCT of an explanation can be computed by reasoning over the prediction scores induced by the GBT model. Next, we demonstrate how MCT can be used to enforce user-defined confidence thresholds, leading to confidence-aware abductive explanations. Finally, we discuss how existing abductive explanations can be refined to obtain stronger confidence guarantees.




\subsection{Computing the MCT of an Explanation}\label{sec:mct}

The previous section defined the Minimum Confidence Threshold (MCT) of an abductive explanation as the prediction score closest to the decision boundary among all instances satisfying the explanation. We now show how this quantity can be computed for explanations generated from gradient-boosted trees.

Recall from Section~\ref{sec:logic_based_xai} that abductive explanations are computed through satisfiability reasoning over the model encoding $G$ and the prediction formula $D$. In particular, the satisfiability test of $\mathcal{X} \cup {G} \cup {\neg D}$ determines whether a subset of feature assignments $\mathcal{X}$ is sufficient to preserve the predicted class. While this test is sufficient for computing traditional abductive explanations, it provides only a binary answer: either a counterexample exists or it does not. Consequently, it does not provide information about the prediction scores associated with the instances satisfying the explanation.

Computing the MCT requires reasoning about the prediction score itself. Recall from Section~\ref{sec:gbts} that, for a concrete instance $\mathbf{x}$, the GBT model produces a prediction score denoted by $S(\mathbf{x})$. In the logical encoding, this score is represented symbolically by the expression
\[
\sum_{h=1}^{m} o_h + \textit{init},
\]

where each variable $o_h$ denotes the output of tree $R_h$. The possible values of the variables $o_h$ are constrained by the model encoding $G$, according to the decision paths followed in each tree.

Given an explanation $\mathcal{E}$, the MCT corresponds to the prediction score closest to the decision boundary among all instances satisfying the feature assignments in $\mathcal{E}$. Therefore, computing the MCT can be formulated as an optimization problem over the symbolic score expression $\sum_{h=1}^{m} o_h + \textit{init}$, subject to the constraints imposed by the model encoding $G$ and the explanation $\mathcal{E}$.

For instances classified as class $+1$, the worst-case confidence corresponds to the smallest prediction score among all instances satisfying the explanation. Therefore, the $\mathrm{MCT}(\mathcal{E})$ can be computed by solving:
\begin{equation}\label{eq:MCT_optimize_1}
\begin{aligned}
\min \quad & \sum_{h=1}^{m} o_h + \textit{init} \\
\text{s.t.} \quad & G,\\
& \mathcal{E}.
\end{aligned}
\end{equation}

Similarly, for instances classified as class $-1$, the worst-case confidence corresponds to the largest score satisfying the explanation. Therefore, the $\mathrm{MCT}(\mathcal{E})$ can be computed by solving:
\begin{equation}\label{eq:MCT_optimize_0}
\begin{aligned}
\max \quad & \sum_{h=1}^{m} o_h + \textit{init} \\
\text{s.t.} \quad & G,\\
& \mathcal{E}.
\end{aligned}
\end{equation}

The constraints encoded by $G$ ensure that the variables $o_h$ correspond to valid leaf outputs of the regression trees, while the constraints in $\mathcal{E}$ fix the feature assignments appearing in the explanation and allow the remaining features to vary.

Observe that these optimization problems directly correspond to the definition of MCT introduced in Section~\ref{sec:conf_margin}. The optimization searches over all instances satisfying the explanation and identifies the prediction score closest to the decision boundary. Consequently, the resulting value represents the weakest confidence guarantee associated with the explanation.


While the MCT provides a quantitative characterization of the confidence guaranteed by an explanation, it can also be used as a target during explanation generation. This allows explanations to be constructed not only to preserve the predicted class, but also to satisfy a desired confidence requirement.

\subsection{Restricting Minimum Confidence Threshold}


Once the MCT of a candidate explanation can be computed, confidence-aware abductive explanations can be generated by imposing a user-defined threshold $\tau$ on the MCT. Intuitively, while traditional abductive explanations only guarantee that the worst-case prediction score remains on the correct side of the decision boundary, confidence-aware explanations additionally require this score to satisfy a minimum confidence requirement.

The optimization problems presented in Equations~\ref{eq:MCT_optimize_1} and~\ref{eq:MCT_optimize_0} provide exactly the quantity needed to enforce such a requirement. Recall that Algorithm~\ref{alg:axp} computes an abductive explanation through an iterative feature-removal process, where a feature is removed whenever the resulting subset of feature assignments continues to guarantee the predicted class. In the confidence-aware setting, the same procedure is retained, but each candidate explanation is additionally evaluated through its MCT. Rather than verifying only whether the explanation preserves the predicted class, we also require its MCT to satisfy a predefined threshold $\tau$.

More specifically, Algorithm~\ref{alg:axp} can be adapted by replacing the traditional unsatisfiability test with the confidence-aware criterion defined in Equation~\ref{eq:caAXP}. During each feature-removal iteration, the MCT of the candidate explanation is computed using the optimization problems presented in Equations~\ref{eq:MCT_optimize_0} and~\ref{eq:MCT_optimize_1}. A feature is removed only if the resulting explanation satisfies the corresponding MCT threshold. Otherwise, the feature remains in the explanation.

Consequently, explanation generation is no longer driven solely by class preservation. Instead, the algorithm searches for explanations that simultaneously guarantee the predicted class and satisfy a user-specified confidence requirement. The resulting confidence-aware explanation procedure is summarized in Algorithm~\ref{alg:caaxp}.

\begin{algorithm}
\caption{Computing a Confidence-Aware Abductive Explanation}
\label{alg:caaxp}
\begin{algorithmic}[1]
\Require{$\mathbf{x}$, $G$, $\tau$, $C(\mathbf{x})$}
\Ensure{Confidence-aware abductive explanation $\mathcal{E}$}
\State $\mathcal{E} \gets \mathbf{x}$
\For{$(f_i=v_i) \in \mathbf{x}$}
    \State $\mathcal{X} \gets \mathcal{E} \setminus \{x_i=v_i\}$
    \State $m \gets \textsc{ComputeMCT}(\mathcal{X})$
    \If{$(C(\mathbf{x})=+1 \land m \geq \tau)$ \textbf{or} $(C(\mathbf{x})=-1 \land m \leq \tau)$}
        \State $\mathcal{E} \gets \mathcal{X}$
\EndIf
\EndFor
\State \Return $\mathcal{E}$
\end{algorithmic}
\end{algorithm}

In Algorithm~\ref{alg:caaxp}, \textsc{ComputeMCT} denotes the optimization procedure defined by Equations~\ref{eq:MCT_optimize_1} and~\ref{eq:MCT_optimize_0}. Algorithm~\ref{alg:caaxp} differs from Algorithm~\ref{alg:axp} only in the criterion used to remove features. Traditional abductive explanations remove a feature whenever class preservation is maintained, whereas confidence-aware abductive explanations additionally require the resulting explanation to satisfy the specified MCT threshold. Consequently, confidence-aware explanations provide stronger guarantees than traditional abductive explanations. Observe that Algorithm~\ref{alg:axp} is recovered as a special case when the threshold coincides with the decision boundary, i.e., when $\tau = 0$.


The key property underlying Algorithm~\ref{alg:caaxp} is the monotonicity of the confidence-aware condition defined in Equation~\ref{eq:caAXP}. Intuitively, once a candidate explanation fails to satisfy the required threshold, removing additional feature assignments cannot restore the guarantee. We formalize this property below.



\begin{lemma}[Monotonicity of the confidence-aware condition]\label{lem:mono}
Let $\mathcal{X}_1$ and $\mathcal{X}_2$ be two subsets of feature assignments such that $\mathcal{X}_1 \subseteq \mathcal{X}_2$. If $\mathcal{X}_1$ satisfies the confidence-aware condition in Equation~\ref{eq:caAXP}, then $\mathcal{X}_2$ also satisfies it.
\end{lemma}

\begin{proof}
Since $\mathcal{X}_1 \subseteq \mathcal{X}_2$, every instance satisfying $\mathcal{X}_2$ also satisfies $\mathcal{X}_1$. Therefore, the set of instances satisfying $\mathcal{X}_2$ is contained in the set of instances satisfying $\mathcal{X}_1$.

First, consider the case $C(\mathbf{x})=+1$. In this case, the MCT is defined as the minimum prediction score among all instances satisfying the explanation. Since $\mathcal{X}_2$ restricts the set of satisfying instances more than $\mathcal{X}_1$, we have
\[
\mathrm{MCT}(\mathcal{X}_2)
=
\min_{\mathcal{X}_2 \subseteq \mathbf{x}'} S(\mathbf{x}')
\geq
\min_{\mathcal{X}_1 \subseteq \mathbf{x}'} S(\mathbf{x}')
=
\mathrm{MCT}(\mathcal{X}_1).
\]
Thus, if $\mathrm{MCT}(\mathcal{X}_1)\geq \tau$, then $\mathrm{MCT}(\mathcal{X}_2)\geq \tau$. Now consider the case $C(\mathbf{x})=-1$. In this case, the MCT is defined as the maximum prediction score among all instances satisfying the explanation. Since the satisfying instances of $\mathcal{X}_2$ form a subset of the satisfying instances of $\mathcal{X}_1$, we have
\[
\mathrm{MCT}(\mathcal{X}_2)
=
\max_{\mathcal{X}_2 \subseteq \mathbf{x}'} S(\mathbf{x}')
\leq
\max_{\mathcal{X}_1 \subseteq \mathbf{x}'} S(\mathbf{x}')
=
\mathrm{MCT}(\mathcal{X}_1).
\]
Thus, if $\mathrm{MCT}(\mathcal{X}_1)\leq \tau$, then $\mathrm{MCT}(\mathcal{X}_2)\leq \tau$. Therefore, in both cases, whenever $\mathcal{X}_1$ satisfies the confidence-aware condition, every superset $\mathcal{X}_2$ also satisfies it. 
\end{proof}

Equivalently, by contraposition, if a set of feature assignments does not satisfy the confidence-aware condition, then none of its subsets can satisfy it. Lemma~\ref{lem:mono} justifies the feature-removal strategy used in Algorithm~\ref{alg:caaxp}: if removing a feature produces a candidate explanation that violates the confidence-aware condition of Equation~\ref{eq:caAXP}, then any further reduction of that candidate will also violate it. Therefore, the feature must remain in the explanation.

\begin{theorem}
Algorithm~\ref{alg:caaxp} returns a confidence-aware abductive explanation.
\end{theorem}

\begin{proof}
The algorithm removes a feature assignment $x_i=v_i$ only when the resulting set satisfies the confidence-aware condition of Equation~\ref{eq:caAXP}. Therefore, every intermediate explanation generated by the algorithm satisfies the required threshold, and so does the final explanation returned by the algorithm.

Now consider any feature assignment $f_i=v_i$ that remains in the explanation $\mathcal{E}$ returned by Algorithm~\ref{alg:caaxp}. When the algorithm tested this feature, it considered a candidate set $\mathcal{X}$ obtained by removing $f_i=v_i$ from the current explanation. Since the feature was not removed, $\mathcal{X}$ did not satisfy the confidence-aware requirement of Equation~\ref{eq:caAXP}. After this iteration, the algorithm may remove additional feature assignments, so the final set $\mathcal{E}\setminus\{f_i=v_i\}$ is a subset of $\mathcal{X}$. By monotonicity of Lemma~\ref{lem:mono}, if $\mathcal{X}$ does not satisfy Equation~\ref{eq:caAXP}, then no subset of $\mathcal{X}$ satisfies Equation~\ref{eq:caAXP}. Therefore, $\mathcal{E}\setminus\{f_i=v_i\}$ also does not satisfy the confidence-aware requirement. Hence, removing any feature assignment from $\mathcal{E}$ invalidates the required guarantee, and $\mathcal{E}$ is minimal. Therefore, $\mathcal{E}$ is a confidence-aware abductive explanation.



\end{proof}

The previous results assume that an appropriate threshold $\tau$ is known beforehand. However, in many situations a user may already have an abductive explanation and wish to obtain a stronger confidence guarantee. This motivates the refinement procedure presented next.

\subsection{Increasing the Confidence of an Abductive Explanation}\label{subsec:how_increase}

Given an abductive explanation $\mathcal{E}$, its MCT can be computed using the optimization procedures introduced in Equations~\ref{eq:MCT_optimize_1} and~\ref{eq:MCT_optimize_0}. This value quantifies the weakest confidence guarantee provided by the explanation and can be used as a starting point for generating explanations with stronger confidence guarantees.

More specifically, let $\mathcal{E}$ be an abductive explanation with MCT equal to $\mu$. Rather than accepting this confidence level, one may seek an alternative explanation that provides a stronger confidence guarantee. This objective can be pursued in two ways. First, one may simply search for any explanation whose MCT is strictly better than $\mu$, that is, an explanation with MCT greater than $\mu$ for instances predicted as class $+1$ or smaller than $\mu$ for instances predicted as class $-1$. In this case, Algorithm~\ref{alg:caaxp} can be adapted by replacing the non-strict threshold conditions with strict inequalities, requiring $m>\mu$ for class $+1$ and $m<\mu$ for class $-1$. Second, one may specify a target confidence threshold $\tau$ satisfying $\tau > \mu$ for class $+1$ or $\tau < \mu$ for class $-1$, and search for an explanation that satisfies this stronger requirement. In both cases, Algorithm~\ref{alg:caaxp} provides the underlying mechanism for generating explanations with improved confidence guarantees. Since the new explanation must satisfy a more restrictive confidence constraint, it may fix additional feature assignments, thereby reducing the variability of the prediction scores among instances satisfying the explanation.


One possible strategy would be to start from the existing abductive explanation $\mathcal{E}$ and iteratively add feature assignments until the desired confidence threshold is reached. Although this approach may appear natural, it does not necessarily produce the smallest explanation satisfying the new confidence requirement. In particular, a different subset of feature assignments, possibly with the same cardinality as the original explanation, may achieve the desired MCT without containing the original explanation as a subset.

For this reason, rather than incrementally extending an existing explanation, we recompute the explanation from scratch under the new confidence constraint. This allows the search procedure to explore alternative combinations of feature assignments and identify explanations that satisfy the desired confidence threshold while remaining as concise as possible.

As an illustrative example, suppose that an abductive explanation $\mathcal{E}$ has $\mathrm{MCT}(\mathcal{E})=-1$. Although the explanation guarantees the predicted class, some instances satisfying it may lie relatively close to the decision boundary. By imposing a stricter threshold, such as $\tau=-2$, Algorithm~\ref{alg:caaxp} searches for a new explanation in which every satisfying instance remains farther from the decision boundary. Consequently, the resulting explanation provides a stronger confidence guarantee.

It is important to note that the refined explanation is not necessarily a superset of the original one. Since multiple minimal explanations may exist for the same prediction, the confidence-aware search procedure may identify a different subset of feature assignments that achieves a higher MCT. In some cases, the new explanation may even contain the same number of features as the original explanation while providing stronger confidence guarantees. Nevertheless, increasing the required confidence generally reduces the set of admissible explanations and tends to produce explanations containing more feature assignments. This reflects a natural trade-off between explanation succinctness and confidence guarantees: stronger guarantees typically require fixing additional features, thereby reducing the variability of the instances satisfying the explanation.





\section{Experiments}
\label{sec:methodology}

The goal of the experimental evaluation is to assess the confidence guarantees provided by traditional abductive explanations, quantify the effects of incorporating Minimum Confidence Threshold (MCT) constraints during explanation generation, and investigate whether stronger confidence guarantees can sometimes be achieved without increasing explanation cardinality. More specifically, we investigate whether traditional abductive explanations provide strong confidence guarantees, how confidence-aware constraints affect explanation size, and whether alternative explanations with stronger guarantees can be obtained without substantially increasing explanation size.

\subsection{Experimental Setup}
\label{sec:experiments}

The experimental evaluation compares traditional abductive explanations with the confidence-aware explanations proposed in this work. For each explained instance, we first compute a traditional abductive explanation and its corresponding MCT. 
We then generate confidence-aware explanations under increasingly restrictive confidence requirements evaluating how these constraints affect the size of the resulting explanations. Moreover, we additionally investigate whether alternative explanations with stronger confidence guarantees can be found without increasing explanation cardinality. Therefore, the experiments were designed to answer the following research questions (\textbf{RQs}):

\begin{itemize}
\item[\textbf{RQ1}] What confidence guarantees are provided by traditional abductive explanations?
\item[\textbf{RQ2}] How do confidence-aware constraints affect the MCT and the size of explanations?
\item[\textbf{RQ3}] Can confidence-aware abductive explanations provide stronger confidence guarantees without substantially increasing explanation size?
\end{itemize}

The proposed explainers were implemented in Python. GBT models were trained using XGBoost\footnote{https://xgboost.ai}, while satisfiability and optimization queries were solved using Z3~\cite{z3solver}. Traditional abductive explanations were generated using Algorithm~\ref{alg:axp}, whereas confidence-aware explanations were generated using Algorithm~\ref{alg:caaxp}. MCT values were computed through the optimization procedures defined in Equations~\ref{eq:MCT_optimize_1} and~\ref{eq:MCT_optimize_0}.

To evaluate the proposed approach, we trained XGBoost models on benchmark datasets obtained from the PMLB repository~\cite{Olson2017PMLB}. The selected datasets are all binary classification problems and were chosen because they are widely used in machine learning research while covering diverse application domains. All trained models achieved an F1-score of at least $0.90$, ensuring that the explanations were generated from accurate predictive models rather than from poorly performing classifiers. Table~\ref{tab:datasets} summarizes the datasets used in the experiments.

\vspace{\intextsep} %
\begin{center} 
\captionof{table}{Summary of the datasets used in the experiments.} \label{tab:datasets}
\begin{tabular}{|l|c|c|c|c|}
\hline
\textbf{Dataset} & \textbf{Instances} & \textbf{Features}  \\
\hline
MAGIC & 19020 & 10  \\
\hline
Adult & 48842 & 14 \\
\hline
Mushroom & 8124 & 22  \\
\hline
Spambase & 4601 & 57  \\
\hline
\end{tabular}
\end{center}
\vspace{\intextsep} 


For each dataset, $100$ randomly selected samples from each class were explained. First, traditional abductive explanations were generated using Algorithm~\ref{alg:axp}. Their corresponding MCT values were then computed using the optimization-based procedure described in Section~\ref{sec:mct}. These explanations serve as the baseline and correspond to the case where only class preservation is required.

Next, confidence-aware abductive explanations were generated using thresholds derived from the prediction score of the explained instance. For an instance with prediction score $S(\mathbf{x})$, the threshold was defined as
\[
\tau = \alpha \cdot S(\mathbf{x}),
\]

where $\alpha \in \{0.25, 0.50, 0.75\}$. These thresholds were chosen to investigate the relationship between confidence guarantees and explanation size under progressively stricter confidence requirements. Thresholds of 100\% were not considered because they would require the explanation to preserve the original prediction score itself. In practice, this would often force the explanation to contain nearly all feature assignments of the original instance.

For each threshold, we measured the resulting MCT and explanation length. In a separate experiment, we investigated whether alternative explanations with strictly better MCT values could be found while maintaining the same explanation cardinality as the original abductive explanation.

\subsection{Results}
\label{sec:results}


To facilitate interpretation, all MCT values reported in this section were transformed using the sigmoid function. Consequently, values closer to 0 indicate stronger confidence guarantees for class $-1$, whereas values closer to 1 indicate stronger confidence guarantees for class $+1$.

To illustrate the effect of confidence-aware constraints on explanation generation, Table~\ref{tab:sample_explanations} presents one example from each class of the Adult dataset. For both instances, the confidence-aware abductive explanation (CA-AXp) was generated using a threshold equal to 50\% of the original prediction score and achieves a substantially improved MCT compared to the corresponding traditional abductive explanation (AXp).

For the instance predicted as class $-1$, the explanation size increases from 7 to 8, while the transformed MCT decreases from 0.45 to 0.35. Since lower values indicate stronger confidence guarantees for class $-1$, the resulting explanation provides a stronger confidence guarantee for all instances satisfying the explanation. For the instance predicted as class $+1$, the explanation size increases from 5 to 6, while the transformed MCT increases from 0.51 to 0.89, indicating a substantial improvement in the confidence guarantee associated with the explanation.


\vspace{\intextsep} %
\begin{center} 
\captionof{table}{Examples of traditional abductive explanations (AXps) and confidence-aware abductive explanations (CA-AXps) for one instance of each class from the Adult dataset (14 features). Prediction scores and MCT values are reported after applying the sigmoid transformation. The column $\sigma(S(\mathbf{x}))$ reports the confidence associated with the explained instance obtained from the sigmoid-transformed prediction score. For each explanation, $|\mathcal{E}|$ denotes the explanation size and $\sigma(\mathrm{MCT}(\mathcal{E}))$ denotes the confidence guarantee associated with the explanation after applying the sigmoid transformation to its MCT.} \label{tab:sample_explanations}
\begin{tabular}{|c|c|c|c|c|c|c|c|}
\hline
&
&
\multicolumn{3}{c|}{\textbf{AXp}}
&
\multicolumn{3}{c|}{\textbf{CA-AXp}}
\\
\cline{3-8}
\textbf{Class}
&
\boldmath$\sigma(S(\mathbf{x}))$
&
\textbf{Explanation}
&
\boldmath$|\mathcal{E}|$
&
\boldmath$\sigma(\mathrm{MCT}(\mathcal{E}))$
&
\textbf{Explanation}
&
\boldmath$|\mathcal{E}|$
&
\boldmath$\sigma(\mathrm{MCT(\mathcal{E})})$
\\
\hline

$-1$
&
0.25
&
\begin{tabular}[c]{@{}l@{}}
$\{$capital-loss = 0,\\
occupation = 4,\\
age = 40,\\
hours-per-week = 46,\\
capital-gain = 0,\\
education-num = 14,\\
relationship = 0$\}$
\end{tabular}
&
7
&
0.45
&
\begin{tabular}[c]{@{}l@{}}
$\{$capital-loss = 0,\\
occupation = 4,\\
age = 40,\\
hours-per-week = 46,\\
capital-gain = 0,\\
marital-status = 2,\\
education-num = 14,\\
relationship = 0$\}$
\end{tabular}
&
8
&
0.35
\\
\hline

$+1$
&
0.98
&
\begin{tabular}[c]{@{}l@{}}
$\{$occupation = 8,\\
age = 28,\\
capital-gain = 0,\\
education-num = 9,\\
relationship = 2$\}$
\end{tabular}
&
5
&
0.51
&
\begin{tabular}[c]{@{}l@{}}
$\{$capital-loss = 0,\\
hours-per-week = 40,\\
capital-gain = 0,\\
marital-status = 4,\\
education-num = 9,\\
relationship = 2$\}$
\end{tabular}
&
6
&
0.89
\\
\hline

\end{tabular}
\end{center}
\vspace{\intextsep} 

Interestingly, increasing the confidence requirement does not always lead to larger explanations. Starting from each abductive explanation, we searched for a confidence-aware abductive explanation with a strictly better MCT, without imposing any specific confidence threshold. In other words, for instances predicted as class $+1$, we searched for explanations whose MCT is greater than that of the original abductive explanation, whereas for instances predicted as class $-1$, we searched for explanations whose MCT is smaller. In many cases, the confidence-aware procedure identifies alternative explanations with exactly the same cardinality as the original abductive explanation while providing stronger confidence guarantees. Table~\ref{tab:same_size} reports the percentage of explained instances for which this phenomenon was observed. Depending on the dataset and class, the percentage ranges from 1\% to 34\%. For example, stronger confidence guarantees were obtained without increasing explanation cardinality for 34\% of the explained instances predicted as class $-1$ in MAGIC and for 34\% of the explained instances predicted as class $+1$ in Spambase. 

\begin{table}[ht]
\centering
\caption{Percentage of instances for which a confidence-aware explanation with a strictly better confidence was found without increasing explanation cardinality.}\label{tab:same_size} 
\begin{tabular}{|l|c|c|} \hline \textbf{Dataset} & \textbf{Class $-1$} & \textbf{Class $+1$} \\ \hline MAGIC & 34\% & 13\% \\ Adult & 25\% & 12\% \\ Mushroom & 5\% & 1\% \\ Spambase & 16\% & 34\% \\ 
\hline 
\end{tabular} 
\end{table}


These results demonstrate that stronger confidence guarantees do not necessarily require larger explanations. In many cases, alternative explanations with the same cardinality provide stronger confidence guarantees. This observation supports the decision to recompute explanations from scratch rather than incrementally extending an existing explanation, as discussed in Section~\ref{subsec:how_increase}. If confidence improvements always required additional feature assignments, incrementally extending an explanation would be sufficient. However, the results indicate that different minimal explanations for the same prediction may provide substantially different confidence guarantees, making a complete recomputation preferable.

The previous examples illustrate how confidence-aware constraints can modify individual explanations and, in some cases, improve confidence guarantees without increasing explanation size. We now turn to the aggregate results across all datasets and confidence thresholds. The results are summarized in Table~\ref{tab:class0_results} and Table~\ref{tab:class1_results}. Table~\ref{tab:class0_results} reports the results for instances predicted as class $-1$, whereas Table~\ref{tab:class1_results} reports the results for instances predicted as class $+1$. Recall that, after applying the sigmoid transformation, improvements correspond to decreasing values in Table~\ref{tab:class0_results} and increasing values in Table~\ref{tab:class1_results}.

\begin{table}[ht]
\centering
\caption{Results for instances predicted as class $-1$. The column $\sigma(S(\mathbf{x}))$ reports the mean sigmoid-transformed prediction score of the explained instances. For each explanation type, $\sigma(\mathrm{MCT}(\mathcal{E}))$ denotes the mean sigmoid-transformed Minimum Confidence Threshold and $|\mathcal{E}|$ denotes the mean explanation length. Lower $\sigma(\mathrm{MCT}(\mathcal{E}))$ values indicate stronger confidence guarantees.}
\label{tab:class0_results}
\resizebox{\textwidth}{!}{
\begin{tabular}{|c|c|cc|cc|cc|cc|}
\hline
&
&
\multicolumn{2}{c|}{\textbf{AXp}}
&
\multicolumn{2}{c|}{\textbf{CA-AXp (25\%)}}
&
\multicolumn{2}{c|}{\textbf{CA-AXp (50\%)}}
&
\multicolumn{2}{c|}{\textbf{CA-AXp (75\%)}}
\\
\cline{3-10}
\textbf{Dataset}
&
\boldmath$\sigma(S(\mathbf{x}))$
&
\boldmath$\sigma(\mathrm{MCT}(\mathcal{E}))$
&
\boldmath$|\mathcal{E}|$
&
\boldmath$\sigma(\mathrm{MCT}(\mathcal{E}))$
&
\boldmath$|\mathcal{E}|$
&
\boldmath$\sigma(\mathrm{MCT}(\mathcal{E}))$
&
\boldmath$|\mathcal{E}|$
&
\boldmath$\sigma(\mathrm{MCT}(\mathcal{E}))$
&
\boldmath$|\mathcal{E}|$
\\
\hline

MAGIC
&
$0.16 \pm 0.13$
&
$0.44 \pm 0.06$
&
$6.48 \pm 1.01$
&
$0.32 \pm 0.09$
&
$6.94 \pm 0.96$
&
$0.25 \pm 0.10$
&
$7.68 \pm 0.89$
&
$0.18 \pm 0.12$
&
$8.38 \pm 0.75$
\\
\hline

Adult
&
$0.23 \pm 0.17$
&
$0.45 \pm 0.04$
&
$7.35 \pm 2.46$
&
$0.36 \pm 0.11$
&
$7.97 \pm 2.19$
&
$0.30 \pm 0.15$
&
$8.89 \pm 1.94$
&
$0.26 \pm 0.16$
&
$9.89 \pm 1.68$
\\
\hline

Mushroom
&
$0.01 \pm 0.02$
&
$0.17 \pm 0.21$
&
$3.21 \pm 1.03$
&
$0.06 \pm 0.10$
&
$3.64 \pm 1.25$
&
$0.03 \pm 0.05$
&
$4.01 \pm 1.80$
&
$0.01 \pm 0.03$
&
$6.13 \pm 1.86$
\\
\hline

Spambase
&
$0.06 \pm 0.08$
&
$0.48 \pm 0.02$
&
$12.45 \pm 2.37$
&
$0.28 \pm 0.08$
&
$14.78 \pm 1.87$
&
$0.15 \pm 0.09$
&
$16.74 \pm 1.82$
&
$0.09 \pm 0.09$
&
$18.47 \pm 1.41$
\\
\hline

\end{tabular}
}
\end{table}

\begin{table}[ht]
\centering
\caption{Results for instances predicted as class $+1$. The column $\sigma(S(\mathbf{x}))$ reports the mean sigmoid-transformed prediction score of the explained instances. For each explanation type, $\sigma(\mathrm{MCT}(\mathcal{E}))$ denotes the mean sigmoid-transformed Minimum Confidence Threshold and $|\mathcal{E}|$ denotes the mean explanation length. Higher $\sigma(\mathrm{MCT}(\mathcal{E}))$ values indicate stronger confidence guarantees.}
\label{tab:class1_results}
\resizebox{\textwidth}{!}{
\begin{tabular}{|c|c|cc|cc|cc|cc|}
\hline
&
&
\multicolumn{2}{c|}{\textbf{AXp}}
&
\multicolumn{2}{c|}{\textbf{CA-AXp (25\%)}}
&
\multicolumn{2}{c|}{\textbf{CA-AXp (50\%)}}
&
\multicolumn{2}{c|}{\textbf{CA-AXp (75\%)}}
\\
\cline{3-10}
\textbf{Dataset}
&
\boldmath$\sigma(S(\mathbf{x}))$
&
\boldmath$\sigma(\mathrm{MCT}(\mathcal{E}))$
&
\boldmath$|\mathcal{E}|$
&
\boldmath$\sigma(\mathrm{MCT}(\mathcal{E}))$
&
\boldmath$|\mathcal{E}|$
&
\boldmath$\sigma(\mathrm{MCT}(\mathcal{E}))$
&
\boldmath$|\mathcal{E}|$
&
\boldmath$\sigma(\mathrm{MCT}(\mathcal{E}))$
&
\boldmath$|\mathcal{E}|$
\\
\hline

MAGIC
&
$0.84 \pm 0.15$
&
$0.59 \pm 0.09$
&
$3.52 \pm 1.12$
&
$0.70 \pm 0.11$
&
$4.03 \pm 0.94$
&
$0.77 \pm 0.14$
&
$4.38 \pm 1.00$
&
$0.81 \pm 0.15$
&
$5.29 \pm 0.95$
\\
\hline

Adult
&
$0.87 \pm 0.14$
&
$0.58 \pm 0.07$
&
$4.86 \pm 1.02$
&
$0.71 \pm 0.11$
&
$5.06 \pm 0.89$
&
$0.79 \pm 0.13$
&
$5.60 \pm 0.86$
&
$0.84 \pm 0.14$
&
$6.69 \pm 0.76$
\\
\hline

Mushroom
&
$0.99 \pm 0.03$
&
$0.62 \pm 0.13$
&
$3.30 \pm 1.71$
&
$0.91 \pm 0.08$
&
$4.65 \pm 1.60$
&
$0.97 \pm 0.05$
&
$6.03 \pm 1.62$
&
$0.99 \pm 0.04$
&
$8.03 \pm 1.72$
\\
\hline

Spambase
&
$0.92 \pm 0.11$
&
$0.52 \pm 0.02$
&
$13.25 \pm 2.39$
&
$0.71 \pm 0.08$
&
$15.15 \pm 1.81$
&
$0.83 \pm 0.11$
&
$17.45 \pm 1.45$
&
$0.89 \pm 0.11$
&
$19.12 \pm 1.44$
\\
\hline

\end{tabular}
}
\end{table}

A first observation is that traditional abductive explanations frequently provide substantially weaker confidence guarantees than the confidence associated with the explained instance itself. In many cases, the transformed MCT values are close to 0.5 even when the corresponding predictions exhibit high confidence. This indicates that the set of instances characterized by an abductive explanation often includes instances lying much closer to the decision boundary than the original sample. Consequently, explanations that simultaneously cover both high-confidence and low-confidence samples tend to be less reliable for interpreting the model's behavior.

The results also show that confidence-aware constraints consistently improve the MCT of the generated explanations. As the confidence threshold increases from 25\% to 75\% of the original prediction score, the corresponding MCT values move progressively farther from the decision boundary, indicating stronger confidence guarantees.

This improvement comes at the cost of larger explanations. In general, stricter confidence requirements require additional feature assignments to restrict the set of instances satisfying the explanation. Consequently, explanation length tends to increase as the threshold becomes more restrictive.

Figure~\ref{fig:confidence_gain_vs_size} provides a complementary view of this trade-off by relating the gain in confidence to the relative increase in explanation length. The horizontal axis reports the percentage of additional feature assignments required with respect to the total number of features in the dataset, whereas the vertical axis reports the corresponding increase in confidence. More precisely, each point in the figure is represented by the pair $(x,y)$, where
\[
x = 100 \cdot \frac{|\mathcal{E}_{CA}|-|\mathcal{E}_{AXp}|}{d}
\]
denotes the percentage increase in explanation length relative to the total number of features (d) in the dataset. Here, $\mathcal{E}_{AXp}$ denotes the traditional abductive explanation of an instance, whereas $\mathcal{E}_{CA}$ denotes the corresponding confidence-aware abductive explanation generated for the same instance and confidence threshold. The vertical coordinate is defined as
\[
y = \sigma(\mathrm{MCT}(\mathcal{E}_{CA}))-\sigma(\mathrm{MCT}(\mathcal{E}_{AXp}))
\]
denotes the corresponding improvement in the sigmoid-transformed MCT. Consequently, points farther to the right correspond to larger increases in explanation length, whereas points higher in the plot correspond to larger confidence gains.

As expected, stronger confidence requirements generally lead to larger explanations. Nevertheless, the increase in explanation length is often modest compared to the confidence gains obtained. This behavior is particularly evident for Spambase, the dataset with the largest number of features in our experiments $(d=57)$. For the 75\% threshold, the average explanation length increases from 13.25 to 19.12, corresponding to an increase of 5.87 features. Despite this absolute increase, the relative growth in explanation length remains close to 10\% of the available features, while the corresponding sigmoid-transformed MCT increases from 0.52 to 0.89. 

More broadly, the curves originate at $(0,0)$, corresponding to the original abductive explanation before any confidence-aware constraint is imposed. Each point therefore represents the additional explanation length required to achieve a given confidence gain relative to the original explanation. Several datasets exhibit substantial confidence improvements with only modest increases in explanation size. For example, in Adult, a relative growth in explanation length below 6\% already produces a confidence gain of about 0.21. Similarly, for Mushroom, a relative increase of approximately 6\% already provides a confidence gain of approximately 0.29. These examples suggest that substantial improvements in confidence can often be achieved with only a modest increase in explanation size. These findings indicate that stronger confidence guarantees do not necessarily require substantially larger explanations, even in high-dimensional datasets.

\begin{figure}[ht]
\centering
\begin{tikzpicture}
\begin{axis}[
    width=\textwidth,
    height=7cm,
    xlabel={Relative increase in explanation length (\% of features)},
    ylabel={Confidence gain},
    xmin=0,
    xmax=25,
    ymin=0,
    ymax=0.45,
    grid=major,
    legend style={
        at={(0.98,0.02)},
        anchor=south east
    }
]

\addplot[
    mark=*,
    thick
]
coordinates {
    (5.1,0.11)
    (8.6,0.18)
    (17.7,0.22)
};
\addlegendentry{MAGIC}

\addplot[
    mark=square*,
    thick
]
coordinates {
    (1.4,0.13)
    (5.3,0.21)
    (13.07,0.26)
};
\addlegendentry{Adult}

\addplot[
    mark=triangle*,
    thick
]
coordinates {
    (6.1,0.29)
    (12.4,0.35)
    (21.5,0.37)
};
\addlegendentry{Mushroom}

\addplot[
    mark=diamond*,
    thick
]
coordinates {
    (3.3,0.19)
    (7.4,0.31)
    (10.3,0.37)
};
\addlegendentry{Spambase}

\end{axis}
\end{tikzpicture}
\caption{Confidence gain versus relative explanation growth for instances predicted as class $+1$. The horizontal axis reports the increase in explanation length normalized by the total number of features in the dataset, while the vertical axis reports the corresponding increase in sigmoid-transformed MCT. Each curve corresponds to confidence thresholds of 25\%, 50\%, and 75\% of the original prediction score and is measured relative to the original abductive explanation (AXp).}
\label{fig:confidence_gain_vs_size}
\end{figure}
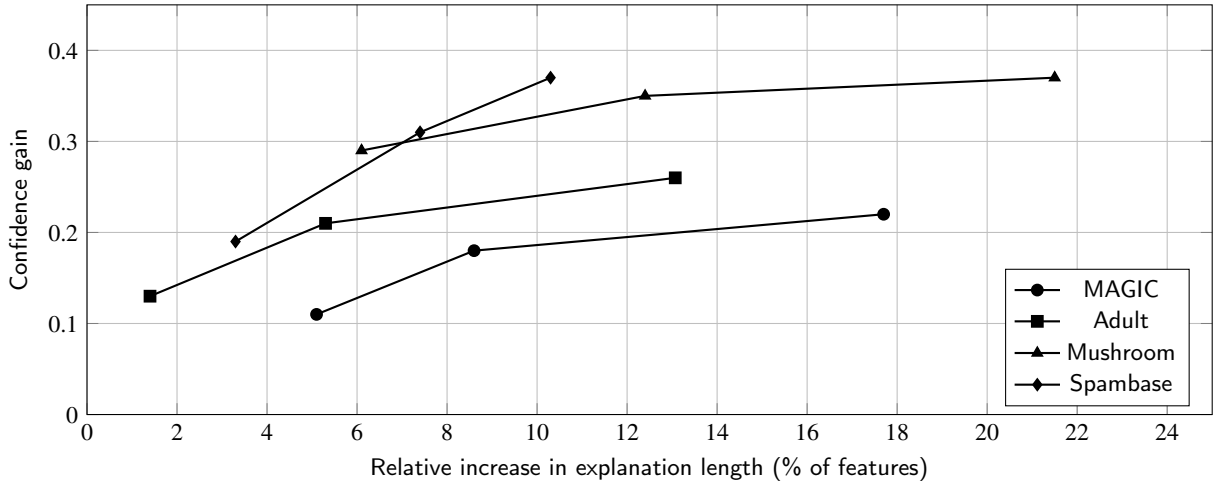

Figure~\ref{fig:confidence_gain_vs_size_neg} shows that the same trade-off observed for class $+1$ also appears for instances predicted as class $-1$. In this case, improvements in confidence correspond to decreases in the sigmoid-transformed MCT, since lower values indicate stronger confidence guarantees. The figure reveals that substantial confidence improvements can often be achieved with relatively small increases in explanation size. For example, in Mushroom, a relative increase of less than 4\% in explanation length already reduces the sigmoid-transformed MCT by approximately 0.14. Similarly, in Spambase, a relative increase of approximately 7.5\% yields a reduction of about 0.33 in sigmoid-transformed MCT. Even for the strictest threshold, Spambase achieves one of the largest confidence improvements observed in the experiments, reducing the average sigmoid-transformed MCT by approximately 0.39 while requiring an increase of only about 10.6\% of the available features.

Taken together with the results for class $+1$, these observations indicate that confidence-aware abductive explanations provide a consistent mechanism for improving confidence guarantees across both classes while requiring only modest increases in explanation size.

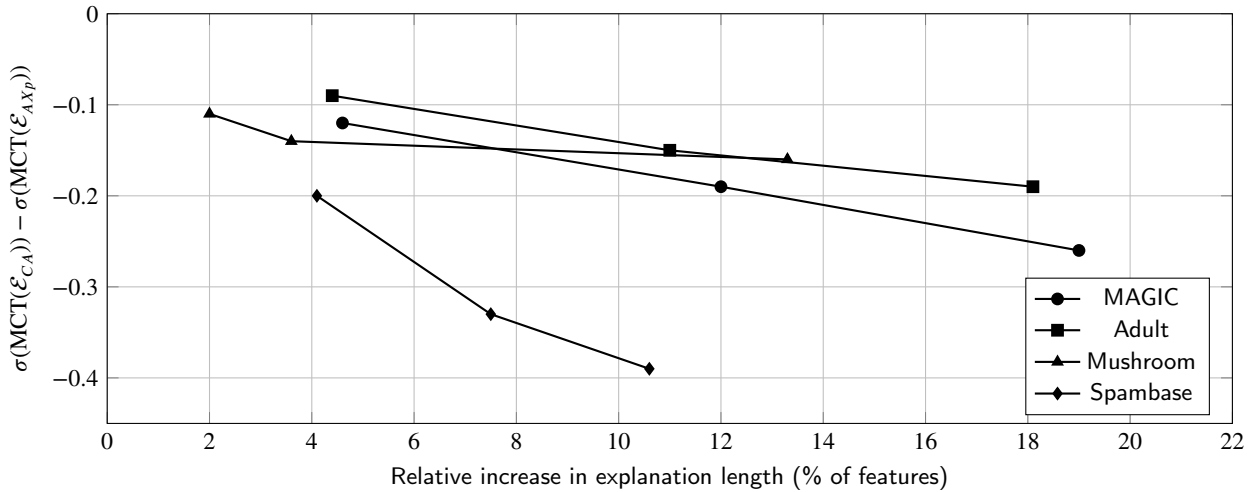
\begin{figure}[ht]
\centering
\begin{tikzpicture}
\begin{axis}[
    width=\textwidth,
    height=7cm,
    xlabel={Relative increase in explanation length (\% of features)},
    ylabel={$\sigma(\mathrm{MCT}(\mathcal{E}_{CA}))-\sigma(\mathrm{MCT}(\mathcal{E}_{AXp}))$},
    xmin=0,
    xmax=22,
    ymin=-0.45,
    ymax=0,
    grid=major,
    legend style={
        at={(0.98,0.02)},
        anchor=south east
    }
]

\addplot[mark=*, thick]
coordinates {
    (4.6,-0.12)
    (12.0,-0.19)
    (19.0,-0.26)
};
\addlegendentry{MAGIC}

\addplot[mark=square*, thick]
coordinates {
    (4.4,-0.09)
    (11.0,-0.15)
    (18.1,-0.19)
};
\addlegendentry{Adult}

\addplot[mark=triangle*, thick]
coordinates {
    (2.0,-0.11)
    (3.6,-0.14)
    (13.3,-0.16)
};
\addlegendentry{Mushroom}

\addplot[mark=diamond*, thick]
coordinates {
    (4.1,-0.20)
    (7.5,-0.33)
    (10.6,-0.39)
};
\addlegendentry{Spambase}

\end{axis}
\end{tikzpicture}
\caption{Confidence gain versus relative explanation growth for instances predicted as class $-1$. The horizontal axis reports the increase in explanation length normalized by the total number of features in the dataset, while the vertical axis is computed as $\sigma(\mathrm{MCT}(\mathcal{E}_{CA}))-\sigma(\mathrm{MCT}(\mathcal{E}_{AXp}))$. Since lower values indicate stronger confidence guarantees for class $-1$, more negative values correspond to larger confidence improvements. Each curve corresponds to confidence thresholds of 25\%, 50\%, and 75\% of the original prediction score and is measured relative to the original abductive explanation (AXp).}
\label{fig:confidence_gain_vs_size_neg}
\end{figure}
\FloatBarrier

Overall, the results confirm the limitation identified in Section~\ref{sec:conf_margin}: preserving the predicted class alone does not guarantee that an explanation preserves the confidence associated with the prediction. Confidence-aware abductive explanations address this issue by explicitly controlling the minimum confidence guaranteed by the explanation. Although stronger guarantees generally require larger explanations, the increase in explanation size is often moderate, and in a substantial number of cases stronger guarantees can be obtained without increasing explanation cardinality.

\section{Conclusions and Future Work}
\label{sec:conclusion}

This paper introduced the concept of Minimum Confidence Threshold (MCT), a measure that quantifies the weakest confidence guarantee provided by an abductive explanation. Building upon this concept, we proposed confidence-aware abductive explanations, which extend traditional abductive explanations by requiring not only preservation of the predicted class but also a user-specified confidence guarantee.

To support the generation of such explanations, we formulated the computation of MCT as an optimization problem and introduced a confidence-aware explanation algorithm capable of producing minimal explanations that satisfy a desired confidence threshold.

Experimental results on several benchmark datasets revealed that traditional abductive explanations often provide substantially weaker confidence guarantees than the confidence associated with the explained instance itself. In contrast, confidence-aware abductive explanations consistently improved the minimum confidence guaranteed by an explanation. Although stronger guarantees generally required larger explanations, the increase in explanation size was often moderate. Moreover, our results showed that explanations with stronger confidence guarantees can frequently be obtained without increasing explanation cardinality, highlighting the importance of searching beyond the original abductive explanation.

These findings are particularly relevant in practical machine learning applications, where predictive confidence is often as important as the predicted class itself. Many modern models provide confidence estimates alongside their predictions, and decision makers routinely use this information when assessing risk and reliability. By incorporating confidence guarantees directly into the explanation process, the proposed framework allows explanations to better reflect not only what decision is made by the model, but also how confidently that decision is supported throughout the region characterized by the explanation.

Overall, the proposed framework provides a principled mechanism for controlling the confidence guarantees associated with logic-based explanations, enabling users to explicitly trade explanation succinctness for stronger confidence guarantees when desired. This contributes to making explanations more informative and trustworthy in real-world scenarios where both correctness and predictive confidence play a central role in decision making.

This work also suggests several directions for future investigation. Although this study focused on gradient boosted trees, the proposed framework can be directly adapted to other machine learning models that provide confidence scores, including deep neural networks. Exploring the behavior of confidence-aware explanations in such models may provide further insights into the relationship between explanation structure and predictive confidence.

Another promising direction is the extension of the proposed framework to multiclass classification problems. In this setting, the notion of minimum confidence threshold could be defined with respect to the confidence margin between the predicted class and its competitors, enabling confidence-aware explanations for more complex decision scenarios.

Finally, the concept of MCT is not restricted to traditional abductive explanations. Future work may investigate how confidence guarantees can be incorporated into other logic-based explanation formalisms, including inflated abductive explanations~\cite{izza2024delivering,rocha2025generalizing,izza2025most}. This could provide a unified framework for reasoning about confidence guarantees across different forms of logic-based explanations.




\printcredits

\FloatBarrier

\bibliographystyle{cas-model2-names}

\bibliography{cas-refs}





\end{document}